\documentclass{arxiv}

\usepackage{times}
\usepackage{mathtools}
\usepackage{microtype}
\usepackage{enumitem}
\usepackage{booktabs}
\usepackage{array}
\setlist{itemsep=2pt,topsep=4pt}
\title[Distribution-Free Feature Learning]{A Distribution-Free Separation Between Learned and Fixed Features}
\altauthor{\Name{Daniel Barzilai} \Email{dbarzilai7@gmail.com}\\\addr{Weizmann Institute of Science}}
\hypersetup{
 pdftitle={A Distribution-Free Separation Between Learned and Fixed Features},
 pdfauthor={Daniel Barzilai},
 pdfkeywords={distribution-free learning, dimension complexity, feature learning, neural networks}
}

\newcommand{\R}{\mathbb R}
\newcommand{\E}{\mathbb E}
\newcommand{\Prob}{\mathbb P}
\newcommand{\ind}{\mathbf 1}
\newcommand{\relu}{\operatorname{ReLU}}
\newcommand{\sign}{\operatorname{sign}}
\newcommand{\dc}{\operatorname{dc}}
\newcommand{\err}{\operatorname{err}}
\newcommand{\Hcal}{\mathcal H}
\newcommand{\Xcal}{\mathcal X}
\newcommand{\Dcal}{\mathcal D}

\newcommand{\norm}[1]{\lVert#1\rVert}
\newcommand{\ip}[2]{\langle #1,#2\rangle}
\newcommand{\cnt}[2]{\lvert#1\rvert_{#2}}
\newcommand{\vx}{\mathbf x}
\newcommand{\vv}{\mathbf v}
\newcommand{\vu}{\mathbf u}
\newcommand{\vz}{\mathbf z}
\newcommand{\vw}{\mathbf w}
\newcommand{\va}{\mathbf a}
\newcommand{\vb}{\mathbf b}
\newcommand{\vtheta}{\boldsymbol\theta}
\newcommand{\vzero}{\mathbf 0}
\newcommand{\vone}{\mathbf 1}
\newcommand{\vxi}{\boldsymbol\xi}
\newcommand{\vphi}{\boldsymbol\phi}

\newcommand{\Cheb}{\operatorname{Cheb}}
\newcommand{\bb}{\overline b^{(0)}}

\begin{document}
\maketitle
\begin{abstract}
The success of neural networks is largely attributed to their ability to learn useful representations from training data. 
However, known results that prove a computational advantage of learned features over fixed ones are restricted to specific input distributions.
We settle the COLT open question of \citet{FKS} by showing that this advantage can persist in the distribution-free setting.

We construct a family of hypothesis classes that can be efficiently learned by a natural gradient-based algorithm on a two-layer ReLU network, but cannot be represented by linear predictors over any fixed feature map with polynomial dimension (and runtime). Specifically, for any $\varepsilon\in(0,1/4)$ and any sufficiently large runtime budget $R$ (depending on $\varepsilon$), there exists a hypothesis class $\Hcal$ that the neural network learns under any input distribution to expected classification error at most $\varepsilon$ in time at most $R$, while the dimension complexity (governing the dimension needed by linear predictors) is at least $R^{\Omega((\log R)^{1/3})}$.
\end{abstract}

\section{Introduction}\label{sec:intro}

What can a neural network trained by gradient-based methods efficiently learn that a linear predictor over a fixed feature map cannot? Existing results establish such a computational advantage in many settings (see Section~\ref{sec:related} for examples). But perhaps surprisingly, these results are all \emph{distribution dependent}, meaning they apply only to a specific input distribution. This prompts the question: can feature learning provably provide a computational advantage in the \emph{distribution-free} setting, where the learning algorithm must succeed under every input distribution?

More concretely, a standard measure of the power of fixed features is the \emph{dimension complexity} $\dc(\Hcal)$ of a
hypothesis class $\Hcal$ on a domain $\Xcal$ \citep{BenDavid}, defined as the smallest dimension $p$ such that there exists a feature map $\varphi:\Xcal\to\R^p$ allowing linear representation
of $\Hcal$ (i.e.,\ $\forall h\in\Hcal,~\exists \vw\in\R^p~\text{such that}~\forall \vx\in \Xcal\,,
h(\vx)=\sign(\langle \vw,\varphi(\vx)\rangle)$). Using this notion, \citet{FKS} raised the following question (informally stated here): 
\begin{center}
\emph{Is there a constant $C$ such that any \textbf{natural} neural network and gradient-based algorithm that learns a class $\Hcal$ to small error under every input distribution requires runtime at least $C\cdot \dc(\Hcal)$?}
\end{center}
The qualification ``natural'' is important. It is well known that the dimension complexity of a hypothesis class can be much larger than its VC dimension \citep{BenDavid}, so empirical risk minimization (ERM) can have low sample complexity even when fixed feature maps require a large dimension. 
However, for standard neural networks, ERM can be computationally hard, and the VC dimension often does not capture the complexity of the learning problem. 
For example, parities over $d$ bits have dimension complexity
$2^{\Omega(d)}$, but can be learned efficiently by Gaussian elimination. There are highly specialized neural networks that can simulate general learning algorithms, and thus learn parities efficiently by simulating Gaussian elimination \citep{AbbeSandon,AbbeKMSS}. However, it is strongly believed that neural networks used in practice, trained by standard gradient-based methods, cannot solve parities efficiently, and suffer a runtime lower bound similar to that of the dimension complexity \citep{shalev2017failures,barak2022hidden}.

A useful surrogate for the learnability of function classes using natural learning algorithms is given by the statistical query (SQ) complexity \citep{Kearns,Feldman17}, which captures the behavior of noise-tolerant algorithms. Recently, \citet{Patel} answered a related question of \citet{FKS} for SQ algorithms in the negative. 
Patel considers disjunctions of $m$ conjunctions, each defined on its own known block of $m^2$ Boolean variables. Their proof combines recent agnostic SQ learners for conjunctions \citep{DKR,CPS} with SQ boosting \citep{AslamDecatur}, yielding, for each integer $m\ge2$, an SQ algorithm with runtime and query complexity $2^{\widetilde O(m^{2/3})}$ and tolerance $2^{-\widetilde O(m^{2/3})}$. Meanwhile, the dimension complexity of the same hypothesis class is $2^{\Omega(m)}$, establishing the separation. However, the SQ algorithm used in their proof is specialized and not a neural network training algorithm. The connection between SQ and algorithms such as SGD used to train neural networks is heuristic, and it is well known that there are cases where the SQ framework fails to capture the complexity of gradient-based learning algorithms (see, e.g., \citet{barzilai26a} for some examples). 

\paragraph{Our contribution.}
We show that a natural gradient-based algorithm training a two-layer ReLU network learns the same hypothesis class considered by \citet{Patel}, establishing a distribution-free separation from linear predictors over fixed features.
Specifically, the following corollary of our main theorem (presented later in Section~\ref{sec:main}) states the separation:
\begin{corollary}\label{cor:separation}
For every $\varepsilon\in(0,1/4)$ and every sufficiently large $R$
(depending on $\varepsilon$), there exists a binary hypothesis class
$\Hcal$ that Algorithm~\ref{alg:training} learns distribution-free to
expected classification error at most $\varepsilon$ in time at most $R$,
whereas $\dc(\Hcal)\ge R^{\Omega((\log R)^{1/3})}$.
\end{corollary}
Our results resolve the open question raised in \citet{FKS} in the negative.

The paper is structured as follows. Section~\ref{sec:setting} defines the hypothesis class, the network, and the training algorithm. Section~\ref{sec:main} states our main theorem. Section~\ref{sec:overview} gives an overview of the proof of Theorem~\ref{thm:main}, which is proved in Section~\ref{sec:proof}. The appendix contains additional proofs and details.

\subsection{Related work}\label{sec:related}

\paragraph{Feature learning versus fixed representations.}
Many works study neural networks in the \emph{kernel regime}, where large width and suitable parametrization ensure the network is roughly linear in a fixed feature map throughout training (e.g., \citep{Jacot,Chizat, lee2019wide,mei2022generalization,barzilaigeneralization}).
Nevertheless, neural networks in practice are often trained with the goal of learning useful features. Computational advantages of feature learning have been shown for sparse parities and staircase functions
\citep{DanielyMalach,Malach,barak2022hidden, AbbeStaircase,AbbeBM22,AbbeBM23},
for Gaussian single- and multi-index models
\citep{Bietti,Damian,Ba,MousaviHosseini,DamianSmoothing,Dandi,LeeOSW},
and in comparisons with neural networks in the kernel regime or random-feature models
\citep{AllenZhuLi,WeiLLM,GhorbaniLazy,GhorbaniWhen,GhorbaniLinearized,YehudaiShamir}.
All of these guarantees assume a specific input distribution, whereas ours holds for every distribution on the input domain. 

\paragraph{Conjunction learning and the SQ separation.}
A classical approach to learning conjunctions is polynomial regression  \citep{KlivansServedio, KKMS}. Recent work gives faster distribution-free agnostic learners \citep{DKR,CPS}, which underpin the SQ separation of
\citet{Patel} discussed above. 
Our proof uses a polynomial approximation technique similar to that used by \citet{DKR}. The algorithms used in these works explicitly construct and fit polynomial
predictors. The challenge here is to obtain the required approximations from trained ReLU neurons.

\section{Setting}\label{sec:setting}

\paragraph{General notation.}
For an integer $m$ we let $[m]:=\{1,\ldots,m\}$, and let $\norm{\cdot}$ be the Euclidean norm. We write $\relu(t)=(t)_+:=\max\{t,0\}$ with the convention that $\relu'(0)=0$, and $\ind[\cdot]$ for indicators. The function $\sign$ acts coordinatewise, with $\sign(0)=0$.

\paragraph{Distribution-free learning.}
For a distribution $\Dcal$ on
$\Xcal$ and a target $h^\star \in \Hcal$, the error of a classifier $\hat h$ is $\err_{\Dcal,h^\star}(\hat h):=\Prob_{\vx\sim\Dcal}[\hat h(\vx)\ne h^\star(\vx)]$. 
An algorithm learns $\Hcal$ \emph{distribution-free to expected error $\varepsilon$} if for every distribution $\Dcal$ on $\Xcal$ and every
$h^\star\in\Hcal$, when given i.i.d.\ samples $(\vx,h^\star(\vx))$ with
$\vx\sim\Dcal$, it outputs a classifier $\hat h$ with
$\E[\err_{\Dcal,h^\star}(\hat h)]\le\varepsilon$.
The expectation is over the samples and the internal randomness of the algorithm.

\paragraph{The hypothesis class.}
Fix $m\ge2$ and let $d:=m^3$. The input domain is $\Xcal_m:=\{0,1\}^d$,
with $m$ blocks $\vx_i\in\{0,1\}^{m^2}$. We allow an arbitrary distribution
on $\Xcal_m$. Write $\overline{\vx}_i:=(\vx_i,\vone-\vx_i)$ for the bits of
block $i$ together with their negations. We call its coordinates
$a\in[2m^2]$ \emph{literals}; literal $a$ of block $i$ is \emph{active} on
$\vx$ if $\overline{x}_{i,a}=1$, so every encoded block has exactly $m^2$
active literals. A
\emph{conjunction} on block $i$ is a function
$\psi(\vx)=\prod_{a\in S}\overline{x}_{i,a}$, where $S\subseteq[2m^2]$ contains at most
one of $k$ and $m^2+k$ for each $k\in[m^2]$ (and $\psi\equiv1$ if $S=\emptyset$). We define the hypothesis class:
\begin{equation}\label{eq:class}
 \Hcal_m:=\left\{\vx\mapsto 2\cdot\ind\left[\bigvee_{i=1}^m\psi_i(\vx)\right]-1:
 \psi_i\text{ is a conjunction on block }i\right\}.
\end{equation}
It is known to have an exponential dimension-complexity lower bound, following the sign-rank construction of \citet{RS}, as stated below.
\begin{proposition}[\citealp{Patel}; \citealp{RS}]\label{prop:dc}
There is an absolute constant $c_1>0$ such that $\dc(\Hcal_m)\ge2^{c_1m}$ for
all $m\ge2$.
\end{proposition}

\paragraph{The network.}
We use a two-layer ReLU network given by $m$ fully connected blocks of $J$ neurons each:
\begin{equation}\label{eq:network}
 F(\vx;\vtheta):=a_0+\sum_{i=1}^m\sum_{j=1}^J
 a_{ij}\relu(\ip{\vw_{ij}}{\overline{\vx}_i}+b_{ij}),
\end{equation}
with parameters $\vtheta=(\vw,\vb,\va)$ consisting of hidden weights
$\vw_{ij}\in\R^{2m^2}$, hidden biases $b_{ij}\in\R$, and output weights
$\va=(a_0,(a_{ij})_{i,j})$. The network has $P=1+mJ(2m^2+2)$ parameters. All parameters are initialized independently as follows:
\begin{equation}\label{eq:initialization}
 \begin{aligned}
 \vw_{ij}^{(0)}&\sim N(\vzero,\kappa^2I_{2m^2}),\qquad
 b_{ij}^{(0)}\sim\operatorname{Unif}[-2m^2,2m^2],\\
 a_{ij}^{(0)}&\sim\operatorname{Unif}\{\pm1/\sqrt{mJ}\},\qquad a_0^{(0)}:=0,
 \end{aligned}
\end{equation}
for some weight scale $\kappa>0$, and we write $\vtheta_0:=(\vw^{(0)},\vb^{(0)},\va^{(0)})$.

\paragraph{Training Algorithm.}
We train with the logistic loss $\ell(y,f):=\log(1+e^{-yf})$ in two stages. Stage~1 takes a single signSGD \citep{Bernstein} step on
the hidden layer using a minibatch of $B$ samples. In computing this step's loss, independent mean-zero Gaussian noise is added to each hidden bias for each sample. This noise is not retained in the updated parameters. Stage~2 fixes the hidden
layer and runs $T$ steps of projected SGD with batch size $1$ on the output
layer. The algorithm uses $n:=B+T$ samples in total. The full algorithm is described in Algorithm~\ref{alg:training}. Its hyperparameters are specified in
Table~\ref{tab:params} in the appendix and depend only on $m$
and $\varepsilon$, not on the distribution or target.
Two-stage training, including a single hidden-layer gradient step, is common in feature-learning theory \citep[e.g.,][]{Ba,Damian,MousaviHosseini},
as is noise injection
\citep[e.g.,][]{NEURIPS2019_62dad6e2,DamianSmoothing}.

\begin{algorithm}[t]
\floatconts{alg:training}
{\caption{Two-stage logistic training of the block ReLU network}}
{
\textbf{Input:} $m\ge2$, accuracy $\varepsilon\in(0,1/4)$, and access to i.i.d.\ labeled samples.

\textbf{Initialize:} Set the hyperparameters of Table~\ref{tab:params}
and draw $\vtheta_0$ by \eqref{eq:initialization}.

\textbf{Stage 1: one noisy signSGD step on the hidden layer.}
\begin{enumerate}[leftmargin=*,itemsep=2pt,topsep=3pt]
\item Draw $B$ i.i.d.\ labeled samples $(\vx^{(u)},y^{(u)})_{u=1}^B$.
Draw $\xi_{uij}\sim N(0,\sigma^2)$ i.i.d. for each $u\in [B], i\in[m], j\in[J]$. Set
$\vxi^{(u)}:=(\xi_{uij})_{i,j}$ and
$\bar\vxi^{(u)}:=(\vzero,\vxi^{(u)},\vzero)$, so only hidden biases are perturbed.
\item For $\widehat L_0(\vtheta):=\frac1B\sum_{u=1}^B
\ell\bigl(y^{(u)},F(\vx^{(u)};\vtheta+\bar\vxi^{(u)})\bigr)$, update all hidden parameters as:
\[
 \vw_{ij}^{(1)}:=\vw_{ij}^{(0)}-\sign(\nabla_{\vw_{ij}}\widehat L_0(\vtheta_0)),
 \quad
 b_{ij}^{(1)}:=b_{ij}^{(0)}-\sign(\partial_{b_{ij}}\widehat L_0(\vtheta_0)).
\]
\end{enumerate}
\textbf{Stage 2: projected SGD on the output layer (hidden layer fixed, no noise).}
\begin{enumerate}[start=3,leftmargin=*,itemsep=2pt,topsep=3pt]
\item Let $\va^{(1)}:=\va^{(0)}$, and write
$\vtheta_t:=(\vw^{(1)},\vb^{(1)},\va^{(t)})$ and
$F_t(\vx):=F(\vx;\vtheta_t)$. For $t=1,\ldots,T$, draw a
fresh labeled sample $(\vx'_t,y'_t)$ and update
\[
 \va^{(t+1)}:=\Pi\bigl(\va^{(t)}-\eta\nabla_{\va}\ell(y'_t,F_t(\vx'_t))\bigr),
\]
where $\Pi$ is Euclidean projection onto $\{\va : \norm{\va}\le R_{\mathrm{out}}\}$.
\item Return $\hat h(\vx):=2\cdot\ind\bigl[F(\vx ; \vtheta^{\mathrm{avg}})\ge0\bigr]-1$ where $\vtheta^{\mathrm{avg}}:=T^{-1}\sum_{t=1}^T\vtheta_t$.
\end{enumerate}
}
\end{algorithm}

\section{Main Results}\label{sec:main}

\begin{theorem}\label{thm:main}
For every $m\ge2$ and $\varepsilon\in(0,1/4)$,
Algorithm~\ref{alg:training}, with the hyperparameters of Table~\ref{tab:params} in the appendix, learns $\Hcal_m$ distribution-free to expected error at most $\varepsilon$, using $n$ samples and $P$ parameters (so $O(Pn)$ arithmetic operations), where
\begin{equation}\label{eq:resources}
 Pn\le\exp\bigl(O(m^{2/3}\log^2(m/\varepsilon))\bigr).
\end{equation}
\end{theorem}

The proof is in Section~\ref{sec:proof}. The theorem should be contrasted with Proposition~\ref{prop:dc}, which shows that $\dc\left(\Hcal_m\right) \geq 2^{\Omega(m)}$. 
Our bound matches that of the SQ algorithm of \citet{Patel}, up to log factors in the exponent. Using this theorem, we prove Corollary~\ref{cor:separation} presented in the introduction, which clarifies the distribution-free separation between gradient-trained neural networks and dimension complexity.

\begin{proofof}{Corollary~\ref{cor:separation}}
Set $m:=\lfloor(\log R)^{4/3}\rfloor$ and take $\Hcal:=\Hcal_m$
on the domain $\Xcal_m$, with input dimension $d=m^3$.
For fixed $\varepsilon$, Theorem~\ref{thm:main} gives
\[
 Pn\le\exp\bigl(O((\log R)^{8/9}
             \log^2((\log R)^{4/3}/\varepsilon))\bigr)=R^{o(1)}.
\]
Here $o(1)$ tends to zero as $R\to\infty$ for each fixed $\varepsilon$.
Since $P\ge d$, total training time, including reading the inputs, is
$O(Pn)=R^{o(1)}\le R$ for all sufficiently large $R$.
For the lower bound on the dimension complexity, Proposition~\ref{prop:dc} gives
$\dc(\Hcal)\ge2^{c_1m}=R^{\Omega((\log R)^{1/3})}$.
\end{proofof}

\section{Proof Overview}\label{sec:overview}

Fix a target $h^\star$ with conjunctions $\psi_1,\ldots,\psi_m$.
For a set of literals $\mathcal Z\subseteq[2m^2]$, write
$\cnt{\vx_i}{\mathcal Z}:=\sum_{a\in\mathcal Z}\overline{x}_{i,a}$ for the number
that are active in $\overline{\vx}_i$.
Let the \emph{violating literals} $V_i$ be the negations of the literals appearing in $\psi_i$.
Then $\cnt{\vx_i}{V_i}$ counts the unsatisfied literals of $\psi_i$, so
$\psi_i(\vx)=1$ iff $\cnt{\vx_i}{V_i}=0$.

The majority of the proof focuses on how linear combinations of neurons can approximate conjunctions after Stage~1, and in turn approximate $h^\star$ well. Stage~2 then solves a convex problem over the features produced by Stage~1, and its convergence follows from the classical projected-SGD bound of \citet[Eq.~(2.19)]{nemirovski2009robust}. 

Suppose that $f^\star_i$ is close to $1$ when $\psi_i$ is satisfied and
close to $0$ when it is violated. Then
\[
 -1+2\sum_{i=1}^mf^\star_i(\vx)\approx
 \begin{cases}
 -1 & \text{if no conjunction is satisfied,}\\
 \ge1 & \text{if some conjunction is satisfied,}
 \end{cases}
\]
so this function has the sign of $h^\star$, with margin about $1$.
It suffices to approximate these functions by linear combinations of
the trained neurons, allowing a small exceptional set of inputs.

To find such an $f^\star_i$, we leverage the observation of \citet[Lemma~3.2]{DKR}, that for any integer cutoff $0<r<m^2$ there is a polynomial $Q_i$ of degree $O(\sqrt r\log(m/\varepsilon))$ such that
\[
 Q_i(\overline{\vx}_i)\begin{cases}
 =1 & \text{if }\cnt{\vx_i}{V_i}=0,\\
 \approx0 & \text{if }1\le\cnt{\vx_i}{V_i}\le r.
 \end{cases}
\]
Now suppose that we find sets $\mathcal Z_i$ such that $V_i \subseteq \mathcal Z_i$ and $\Prob_{\vx}[\psi_i(\vx)=1,\ \cnt{\vx_i}{\mathcal Z_i}>r]$ is small for every block $i$. We may then define
\begin{equation}\label{eq:poly-approx}
 f_i^\star(\vx):=\ind[\cnt{\vx_i}{\mathcal Z_i}\le r]\,Q_i(\overline{\vx}_i),
\end{equation}
so that the choices of $Q_i$ and $\mathcal Z_i$ give $f_i^\star(\vx)\approx\ind[\psi_i(\vx)=1]$ whenever $\cnt{\vx_i}{\mathcal Z_i}\le r$.
We will show in Section~\ref{sec:targets} that we can indeed find such sets $\mathcal Z_i$.

The cutoff $r$ largely determines the runtime. A larger choice of $r$ makes finding a suitable set $\mathcal Z_i$ easier, but it also increases the polynomial degree of $Q_i$ and the width needed to approximate \eqref{eq:poly-approx}. 
The proof obtains a suitable $\mathcal Z_i$ from $O(m^2/r)$ batch samples.
In our construction, representing the function
$\ind[\cnt{\vx_i}{\mathcal Z_i}\le r]$ for a suitable $\mathcal Z_i$
incurs a width cost $B^{O(m^2/r)}$, where $B$ is the batch size in Stage~1.
The polynomial approximation requires width
exponential in its degree $O(\sqrt r\log(m/\varepsilon))$.
Thus, up to logarithmic factors, these costs contribute $m^2/r$ and $\sqrt r$ to $\log J$. 
Taking $r\approx m^{4/3}$ balances them and gives
$\log J=O(m^{2/3}\log^2(m/\varepsilon))$, which also bounds the logarithm of the total training time.

The proof is structured as follows:
\begin{itemize}
 \item Section~\ref{sec:counters} relates the neurons trained after Stage~1 to literal counts.
 \item Section~\ref{sec:targets} shows that, under any input distribution, the batch contains subsets defining suitable sets $\mathcal Z_i$, and that the resulting products approximate the target DNF.
 \item Section~\ref{sec:representation} shows that linear combinations of these neurons can approximate $f_i^\star$ defined in \eqref{eq:poly-approx}.
 \item Section~\ref{sec:output} combines these approximations and analyzes
 the second-stage optimization to complete the learning guarantee.
\end{itemize}

\section{Proof of Theorem~\ref{thm:main}}\label{sec:proof}

Fix a target $h^\star\in\Hcal_m$, with conjunctions $\psi_1,\ldots,\psi_m$ and violating
literals $V_1,\ldots,V_m$ as in Section~\ref{sec:overview}, so that
$h^\star(\vx)=2\cdot \ind[\bigvee_i\psi_i(\vx)]-1$, and fix a distribution $\Dcal$ on
$\Xcal_m$.
Unless explicitly subscripted, probabilities $\Prob$ and expectations $\E$ are over all
remaining randomness. We write $\Prob_{\vx}$ and $\E_{\vx}$ for probability
and expectation over a fresh input $\vx\sim\Dcal$, holding the algorithm's
random quantities fixed. The proof
uses an accuracy budget $\varepsilon_0$ and a failure-probability budget
$\delta$,
\begin{equation}\label{eq:proof-budgets}
 \varepsilon_0:=\frac{\varepsilon}{1000\log(1/\varepsilon)},
 \qquad \delta:=\frac{\varepsilon}{1000}.
\end{equation}
For the analysis, define the hinge loss $\ell_{\mathrm{hinge}}(y,f):=(1-yf)_+$.
We use the cutoff
$r:=\min\{\lfloor m^2/2\rfloor,\lfloor m^{4/3}\rfloor\}$ and the subset-size bound
$s:=\lfloor2m^2/(r+1)\rfloor$.

\subsection{Stage 1 turns some neurons into counters}\label{sec:counters}

We first describe the neurons that will be used to approximate conjunctions.
\begin{definition}[Useful neuron]\label{def:useful}
Let $\mathcal U\subseteq[B]$ be nonempty, with $|\mathcal U|\le s$ and $y^{(u)}=+1$ for all
$u\in \mathcal U$. Let $E_{ij}(\mathcal U)$ be the event that the noise of neuron $(i,j)$
satisfies
\[
 \xi_{uij}>t_+\ \text{ for every }u\in \mathcal U
 \qquad\text{and}\qquad
 \xi_{uij}\le t_-\ \text{ for every }u\in[B]\setminus \mathcal U,
\]
where $t_+:=m^2+2+\frac18$ and $t_-:=m^2-r-1-\frac18$. Write
$\bb_{ij}:=b_{ij}^{(0)}+m^2+1$ for the shifted initial bias, which is uniform on
$[1-m^2,3m^2+1]$. Neuron $(i,j)$ is
\emph{$\mathcal U$-useful} if $E_{ij}(\mathcal U)$ holds, $\bb_{ij}\in[-1,r+2]$,
$\norm{\vw_{ij}^{(0)}}\le1/(16m)$, and
$a_{ij}^{(0)}>0$.
\end{definition}
\begin{lemma}[Activation probability]\label{lem:activation}
Given any Stage-1 batch, fix a block $i$ and a subset $\mathcal U$
as in Definition~\ref{def:useful}. Conditional on the batch, the
events $E_{ij}(\mathcal U)$, $j\in[J]$, are mutually independent
and independent of the initialization, with common probability
$p(\mathcal U)$ satisfying
\begin{equation}\label{eq:noise-probability}
 p(\mathcal U)\ge B^{-4C_0^2s},
\end{equation}
where $C_0$ is the constant in Table~\ref{tab:params}.
Moreover, each neuron $(i,j)$ is $\mathcal U$-useful with conditional
probability at least
\[
 \frac{r+3}{16m^2}B^{-4C_0^2s}\ge\exp(-O(m^{2/3}\log(m/\varepsilon))).
\]
\end{lemma}
The proof (Appendix~\ref{app:deferred}) is a Gaussian tail calculation.
The width $J$ is chosen large enough to compensate for the rarity of
useful neurons (shown later in Proposition~\ref{prop:features}).

The next lemma shows how the features produced by useful neurons after Stage~1 depend on the input mainly through the number of active literals in a set $\mathcal Z_i$.
\begin{lemma}[Useful neurons give counters]\label{lem:sign-step}
Let $\mathcal U$ be as in Definition~\ref{def:useful}, and let
\[
 \mathcal Z_i(\mathcal U):=\{a\in[2m^2]:\overline{x}_{i,a}^{(u)}=0\ \forall u\in\mathcal U\},
 \qquad \mathcal Z_i(\emptyset):=[2m^2]
\]
be the set of literals of block $i$ that are inactive in every 
$u\in\mathcal U$.
If a neuron $(i,j)$ is
$\mathcal U$-useful, then
\begin{equation}\label{eq:useful-affine}
 \forall\vx\in\Xcal_m ~,\qquad \ip{\vw_{ij}^{(1)}}{\overline{\vx}_i}+b_{ij}^{(1)}
 =\bb_{ij}-\cnt{\vx_i}{\mathcal Z_i(\mathcal U)}
 +\ip{\vw_{ij}^{(0)}}{\overline{\vx}_i}.
\end{equation}
\end{lemma}
\begin{proof}
Assume that neuron $(i,j)$ is $\mathcal U$-useful, and write
$\mathcal Z=\mathcal Z_i(\mathcal U)$.
On sample $u \in [B]$,
\begin{align}\label{eq:preactivation}
 \ip{\vw_{ij}^{(0)}}{\overline{\vx}_i^{(u)}}
 +b_{ij}^{(0)}+\xi_{uij}
 =\bb_{ij}-m^2-1
  +\ip{\vw_{ij}^{(0)}}{\overline{\vx}_i^{(u)}}+\xi_{uij}.
\end{align}
Usefulness gives $-1\le\bb_{ij}\le r+2$ and
$|\ip{\vw_{ij}^{(0)}}{\overline{\vx}_i^{(u)}}|
\le m\norm{\vw_{ij}^{(0)}}\le1/16$.
Together with the noise inequalities in
Definition~\ref{def:useful}, these imply that \eqref{eq:preactivation} is $>1/16$ for $u\in\mathcal U$ and $\le-1/16$ for $u\notin\mathcal U$.
Thus, such a neuron's ReLU derivative is $\ind[u\in\mathcal U]$.
Since $a_{ij}^{(0)}=1/\sqrt{mJ}$ and
$y^{(u)}=+1$ on $\mathcal U$, 
\[
 \partial_{w_{ij,a}}\widehat L_0(\vtheta_0)
 =-\frac1{B\sqrt{mJ}}\sum_{u\in \mathcal U}
 \frac{\overline{x}_{i,a}^{(u)}}{1+e^{F(\vx^{(u)};\vtheta_0+\bar\vxi^{(u)})}}.
\]
The denominator is positive, so this derivative is negative
exactly when $a\notin\mathcal Z$, and zero otherwise.
The bias derivative has the same formula with numerator $1$,
so it is negative because $\mathcal U\ne\emptyset$.
Consequently, signSGD gives
\[
 w_{ij,a}^{(1)}
 =w_{ij,a}^{(0)}+\ind[a\notin\mathcal Z],
 \qquad b_{ij}^{(1)}=b_{ij}^{(0)}+1.
\]
Finally, $\sum_{a\notin\mathcal Z}\overline{x}_{i,a}
=m^2-\cnt{\vx_i}{\mathcal Z}$, since every encoded block has
exactly $m^2$ active literals. Hence
\[
 \begin{aligned}
 \ip{\vw_{ij}^{(1)}}{\overline{\vx}_i}+b_{ij}^{(1)}
 &=\ip{\vw_{ij}^{(0)}}{\overline{\vx}_i}
   +b_{ij}^{(0)}+1+m^2-\cnt{\vx_i}{\mathcal Z}
   =\bb_{ij}-\cnt{\vx_i}{\mathcal Z}
   +\ip{\vw_{ij}^{(0)}}{\overline{\vx}_i}.
 \end{aligned}
\]
\end{proof}
\subsection{Approximating each conjunction}\label{sec:targets}
\begin{definition}[Good set of literals]\label{def:good}
A set $\mathcal Z\subseteq[2m^2]$ is \emph{good for block $i$} if
\begin{equation}\label{eq:good-set}
 V_i\subseteq\mathcal Z
 \qquad\text{and}\qquad
 \Prob_{\vx}\bigl[\psi_i(\vx)=1,\ \cnt{\vx_i}{\mathcal Z}>r\bigr]
 \le\frac{\varepsilon_0}{m}.
\end{equation}
\end{definition}

\begin{lemma}[Positive sample compression]\label{lem:compression}
With probability at least $1-\delta$ over the Stage-1 batch, for every
$i\in[m]$ there is a subset $\mathcal U_i\subseteq[B]$ with
$|\mathcal U_i|\le s$ and $\psi_i(\vx^{(u)})=1$ for every $u\in\mathcal U_i$,
such that $\mathcal Z_i(\mathcal U_i)$ is good for block $i$.
\end{lemma}
The proof is given in Appendix~\ref{app:compression}. The following lemma is inspired by \citet{DKR}.

\begin{lemma}[Polynomial approximation]\label{lem:polynomial}
Let $k:=\lceil10m^{2/3}\log(m/\varepsilon)\rceil$.
For each $i\in[m]$, there is a
polynomial $Q_i:\{0,1\}^{2m^2}\to\R$ of degree at most $k$, depending only
on $\cnt{\vz}{V_i}:=\sum_{a\in V_i}z_a$, with $|Q_i|\le(4(2m^2+1))^{k+1}$ and
\begin{equation}\label{eq:Q-props}
 Q_i(\vz)=1\ \text{ if }\cnt{\vz}{V_i}=0,
 \qquad
 |Q_i(\vz)|\le\frac{\varepsilon_0}{m}\ \text{ if }1\le\cnt{\vz}{V_i}\le r.
\end{equation}
\end{lemma}
The proof is given in Appendix~\ref{app:chebyshev}. For polynomials $Q_i$ from Lemma~\ref{lem:polynomial} and sets
$\mathcal Z_i\subseteq[2m^2]$, define
\begin{equation}\label{eq:ideal-block}
 f^\star_i(\vx):=\ind\bigl[\cnt{\vx_i}{\mathcal Z_i}\le r\bigr]Q_i(\overline{\vx}_i),
 \qquad i\in[m].
\end{equation}

\begin{lemma}[Approximating the conjunctions]\label{lem:targets}
Suppose that for every $i\in[m]$, the set $\mathcal Z_i$ is good for block $i$. If functions $\widehat f_i$, $i\in[m]$, satisfy
$\sum_{i=1}^m\norm{\widehat f_i-f^\star_i}_{L_1(\Dcal)}<2\varepsilon_0$, then
\[
 \E_{\vx}\ell_{\mathrm{hinge}}\Bigl(h^\star(\vx),
       -1+2\sum_{i=1}^m\widehat f_i(\vx)\Bigr)<8\varepsilon_0.
\]
\end{lemma}
\begin{proof}
The containment $V_i\subseteq\mathcal Z_i$ and \eqref{eq:Q-props} imply
that 
\begin{equation}\label{eq:block-target-props}
 \E_{\vx}\bigl|f^\star_i(\vx)-\ind[\psi_i(\vx)=1]\bigr|
 \le\frac{\varepsilon_0}{m}
    +\Prob_{\vx}\bigl[\psi_i(\vx)=1,\ \cnt{\vx_i}{\mathcal Z_i}>r\bigr]
 \le\frac{2\varepsilon_0}{m}.
\end{equation}
Moreover,
$\ell_{\mathrm{hinge}}(h^\star(\vx),-1+2\sum_i\ind[\psi_i(\vx)=1])=0$.
Since $\ell_{\mathrm{hinge}}$ is
$1$-Lipschitz in its second argument, \eqref{eq:block-target-props} and
the triangle inequality give
\[
 \E_{\vx}\ell_{\mathrm{hinge}}\Bigl(h^\star(\vx),
       -1+2\sum_{i=1}^m\widehat f_i(\vx)\Bigr)
 \le 2\sum_{i=1}^m\norm{\widehat f_i-f^\star_i}_{L_1(\Dcal)}
       +4\varepsilon_0<8\varepsilon_0.
\]
\end{proof}

\subsection{Counters represent a cutoff times a polynomial}\label{sec:representation}

We now show how linear combinations of Stage-1 neurons approximate $f_i^\star$. First, we represent the polynomial using the Gaussian
initialization; then we incorporate the cutoff by combining neurons with
different biases. For
$\vw\in\R^{2m^2}$ and $\vz\in\{0,1\}^{2m^2}$, define the normalized random
ReLU feature
\begin{equation}\label{eq:rho}
 \rho_{\vw}(\vz):=\frac{\relu(\ip{\vw}{\vz})}{\kappa(2m^2+1)}.
\end{equation}

Set $\Lambda:=m^{100k}$ and
$M:=1/(16m\kappa)=\Theta(m\log(m/\varepsilon))$. We use only initial
weights satisfying $\norm{\vw}\le\kappa M=1/(16m)$, as required by
Definition~\ref{def:useful}. This keeps their contribution to the neuron's
input small, while excluding only a small Gaussian tail.

The following lemma represents each $Q_i$ exactly as an expectation of
the Gaussian ReLU features in \eqref{eq:rho}. Using the Gaussian initialization and various ReLU identities, this lemma will be critical in finding linear combinations of neurons that can approximate $f_i^\star$ after Stage~1.

\begin{lemma}[Centered Gaussian ReLU representation]\label{lem:gaussian}
Let $Q:\{0,1\}^{2m^2}\to\R$ be a polynomial of degree at most $k$ with
$|Q|\le(4(2m^2+1))^{k+1}$, and let $\vw\sim N(\vzero,\kappa^2I_{2m^2})$. There is a
function $\alpha_Q:\R^{2m^2}\to\R$ such that
\begin{enumerate}[label=(\roman*)]
\item $\alpha_Q(\vw)=0$ if $\norm{\vw}>\kappa M$, and $|\alpha_Q(\vw)|\le\Lambda M$ otherwise;
\item $\E[\alpha_Q(\vw)]=0$;
\item for every $\vu\in\{0,1\}^{m^2}$ and $\vz:=(\vu,\vone-\vu)$,
$\E[\alpha_Q(\vw)\rho_{\vw}(\vz)]=Q(\vz)$.
\end{enumerate}
\end{lemma}
Appendix~\ref{app:gaussian} proves the lemma using the kernel of
Gaussian ReLU features \citep{ChoSaul}, its Fourier expansion on the cube,
and rotational invariance.

The block target $f_i^\star$ is zero whenever
$\cnt{\vx_i}{\mathcal Z_i}>r$. The following identity helps us approximate
this behavior using linear combinations of ReLU neurons with different biases.

\begin{lemma}[Grid identity]\label{lem:grid}
For every integer $q\ge0$ and every real $\zeta$ with $|\zeta|\le1/16$,
\begin{equation}\label{eq:grid-identity}
 \sum_{j=0}^r\bigl(\relu(j-q+\zeta)-\relu(j-\tfrac12-q+\zeta)\bigr)
 =\tfrac12(r-q)_++\ind[q\le r]\relu(\zeta).
\end{equation}
\end{lemma}
The proof is given in Appendix~\ref{app:grid}.
The constants $j$ and $j-\tfrac12$ in Lemma~\ref{lem:grid}, for $j=0,\ldots,r$,
form the grid $G:=\{-\frac12,0,\frac12,\ldots,r\}$. We will apply this identity
to combine trained neurons whose shifted biases $\bb_{ij}$ are close to these
values. Since these biases are random, we use those within distance
$\tau:=\varepsilon_0\kappa/(100m\Lambda M)$ of $G$. This makes the approximation
error from replacing each bias by its nearest grid point small.
Let $\Theta_\tau:=\bigcup_{\gamma\in G}[\gamma-\tau,\gamma+\tau]$;
these intervals are disjoint and contained in $[-1,r+2]$.
For $b\in\Theta_\tau$, set $\beta(b)=+1$ if the nearest grid point is an
integer and $-1$ otherwise, matching the signs in \eqref{eq:grid-identity};
set $\beta(b)=0$ outside $\Theta_\tau$.

For each neuron, write
\begin{equation}\label{eq:trained-activation}
 \phi_{ij}(\vx):=\relu(\ip{\vw_{ij}^{(1)}}{\overline{\vx}_i}+b_{ij}^{(1)}).
\end{equation}

\begin{lemma}[One neuron in expectation]\label{lem:one-neuron}
Fix any realization of the Stage-1 batch. Let $\mathcal U$ be as in Definition~\ref{def:useful}, $\mathcal Z=\mathcal Z_i(\mathcal U)$, and $Q$ as in
Lemma~\ref{lem:gaussian}. For $j\in[J]$ define
\[
 v_{ij}:=\frac{4m^2}{\kappa(2m^2+1)p(\mathcal U)\tau}\,
 \ind[(i,j)\text{ is }\mathcal U\text{-useful}]\,
 \ind[\bb_{ij}\in\Theta_\tau]\,
 \beta(\bb_{ij})\,\alpha_Q(\vw_{ij}^{(0)}).
\]
Then $|v_{ij}|\le v_{\max}$, where
\begin{equation}\label{eq:vbar}
 v_{\max}:=\frac{2\Lambda M}{\kappa\tau}B^{4C_0^2s},
\end{equation}
and for every $\vx\in\Xcal_m$, with expectation over the initialization and noise,
\[
 \Bigl|\E[v_{ij}\phi_{ij}(\vx)]-\ind\bigl[\cnt{\vx_i}{\mathcal Z}\le r\bigr]Q(\overline{\vx}_i)\Bigr|
 \le\frac{\varepsilon_0}{100m}.
\]
\end{lemma}
\begin{proof}
\eqref{eq:vbar} follows from $p(\mathcal U)\ge B^{-4C_0^2s}$, 
Lemma~\ref{lem:gaussian}(i), $|\beta| \leq 1$, and the definition of a useful neuron. Fix $\vx$.
Since $\kappa M=1/(16m)$, Definition~\ref{def:useful} gives
\[
 \begin{aligned}
 \ind\left[(i,j)\text{ is }\mathcal U\text{-useful}\right]
 &=\ind\left[E_{ij}(\mathcal U)\right]\ind\left[a_{ij}^{(0)}>0\right]
 \cdot\ind\left[\norm{\vw_{ij}^{(0)}}\le\kappa M\right]
 \ind\left[\bb_{ij}\in[-1,r+2]\right].
 \end{aligned}
\]
In the definition of $v_{ij}$, the norm indicator can be removed because
$\alpha_Q(\vw_{ij}^{(0)})=0$ when
$\norm{\vw_{ij}^{(0)}}>\kappa M$. The bias indicator can also be removed:
$\ind\left[\bb_{ij}\in\Theta_\tau\right]$ already enforces
$\bb_{ij}\in[-1,r+2]$, since $\Theta_\tau\subseteq[-1,r+2]$.
On the useful-neuron event, \eqref{eq:useful-affine} and
\eqref{eq:trained-activation} give
$\phi_{ij}(\vx)=\relu(\bb_{ij}-\cnt{\vx_i}{\mathcal Z}+\ip{\vw_{ij}^{(0)}}{\overline{\vx}_i})$;
outside that event, $v_{ij}=0$. Substituting this formula gives
\begin{equation}\label{eq:weighted-neuron}
 \begin{aligned}
 v_{ij}\phi_{ij}(\vx)
 &=\frac{4m^2\alpha_Q(\vw_{ij}^{(0)})}{\kappa(2m^2+1)p(\mathcal U)\tau}
   \ind[E_{ij}(\mathcal U)]\ind[a_{ij}^{(0)}>0]\\
 &\qquad\cdot\ind[\bb_{ij}\in\Theta_\tau]\beta(\bb_{ij})
   \relu\bigl(\bb_{ij}-\cnt{\vx_i}{\mathcal Z}+\ip{\vw_{ij}^{(0)}}{\overline{\vx}_i}\bigr).
 \end{aligned}
\end{equation}
The indicators $\ind[E_{ij}(\mathcal U)]$ and $\ind[a_{ij}^{(0)}>0]$
are independent of each other and of $(\bb_{ij},\vw_{ij}^{(0)})$, and have
expectations $p(\mathcal U)$ and $1/2$, respectively. Taking expectation over these two indicators replaces their product
by $p(\mathcal U)/2$.
The bias $\bb_{ij}$ is uniform on an interval of length $4m^2$, so averaging
over it means integrating with density $1/(4m^2)$. Thus,
\[
 \begin{aligned}
 \E[v_{ij}\phi_{ij}(\vx)]
 &=\frac1{2\kappa(2m^2+1)\tau}\\
 &\quad\cdot\E_{\vw_{ij}^{(0)}}\left[\alpha_Q(\vw_{ij}^{(0)})
 \sum_{\gamma\in G}\beta(\gamma)
 \int_{\gamma-\tau}^{\gamma+\tau}
 \relu\bigl(b-\cnt{\vx_i}{\mathcal Z}+\ip{\vw_{ij}^{(0)}}{\overline{\vx}_i}\bigr)\,db\right],
 \end{aligned}
\]
where the remaining expectation is over $\vw_{ij}^{(0)}\sim N(\vzero,\kappa^2I_{2m^2})$.
Replacing $b$ by $\gamma$ inside each ReLU changes this expression by at
most $2(r+1)\Lambda M\tau/(\kappa(2m^2+1))\le\varepsilon_0/(100m)$,
since ReLU is $1$-Lipschitz. After this replacement, the definition of
$\beta$ and Lemma~\ref{lem:grid} give
\[
 \begin{aligned}
 &\frac1{\kappa(2m^2+1)}\,
 \E_{\vw_{ij}^{(0)}}\left[\alpha_Q(\vw_{ij}^{(0)})
 \left(\tfrac12(r-\cnt{\vx_i}{\mathcal Z})_+
       +\ind[\cnt{\vx_i}{\mathcal Z}\le r]\relu\bigl(\ip{\vw_{ij}^{(0)}}{\overline{\vx}_i}\bigr)\right)\right]\\
 &\qquad=\ind[\cnt{\vx_i}{\mathcal Z}\le r]\,\E_{\vw_{ij}^{(0)}}[\alpha_Q(\vw_{ij}^{(0)})\rho_{\vw_{ij}^{(0)}}(\overline{\vx}_i)]\\
 &\qquad=\ind[\cnt{\vx_i}{\mathcal Z}\le r]\,Q(\overline{\vx}_i).
 \end{aligned}
\]
Here the grid identity applies because
$|\ip{\vw_{ij}^{(0)}}{\overline{\vx}_i}|\le1/16$ on the support of $\alpha_Q$, and the count-only
term cancels because $\E_{\vw_{ij}^{(0)}}[\alpha_Q(\vw_{ij}^{(0)})]=0$.
The last equality uses Lemma~\ref{lem:gaussian}(iii).
The error from replacing $b$ by $\gamma$ is at most $\varepsilon_0/(100m)$,
proving the claim.
\end{proof}

\subsection{Completing the proof}\label{sec:output}

For the analysis of the output layer, collect the fixed hidden activations as
\[
 \vphi(\vx):=(1,(\phi_{ij}(\vx))_{i,j})\in\R^{1+mJ}.
\]
Then $F_t(\vx)=\ip{\va^{(t)}}{\vphi(\vx)}$ for every $t\ge1$.

\begin{proposition}[Approximation by the trained layer]\label{prop:features}
Suppose that the event in Lemma~\ref{lem:compression} holds. With probability
at least $1-\delta$ over the initialization and the noise, there is an
output vector $\va^*$ with $\norm{\va^*}\le\sqrt2$ such that
\begin{align}\label{eq:hinge_bound}
 \E_{\vx}\ell_{\mathrm{hinge}}(h^\star(\vx),\ip{\va^*}{\vphi(\vx)})<8\varepsilon_0.
\end{align}
\end{proposition}
\begin{proof}
Choose subsets $\mathcal U_i$ as in Lemma~\ref{lem:compression}, write
$\mathcal Z_i=\mathcal Z_i(\mathcal U_i)$, take $Q_i$ from
Lemma~\ref{lem:polynomial}, and define $f^\star_i$ by
\eqref{eq:ideal-block}. Since $\psi_i(\vx^{(u)})=1$ implies
$y^{(u)}=h^\star(\vx^{(u)})=+1$, each nonempty $\mathcal U_i$ contains
at most $s$ samples, all with label $+1$, as required to apply
Lemma~\ref{lem:one-neuron}. If $\mathcal U_i=\emptyset$, then
$f^\star_i=0$, since every encoded block has $m^2>r$ active literals.
For each $i$ with $\mathcal U_i\ne\emptyset$, let $v_{ij}$ be the
coefficients of Lemma~\ref{lem:one-neuron} with $\mathcal U=\mathcal U_i$
and $Q=Q_i$; otherwise set $v_{ij}=0$. Define
$\widehat f_i:=J^{-1}\sum_{j=1}^Jv_{ij}\phi_{ij}$. By
Lemma~\ref{lem:one-neuron} (or $\widehat f_i=f^\star_i=0$ for an empty subset),
for every $\vx\in\Xcal_m$,
\begin{equation}\label{eq:block-mean-approximation}
 |\E\widehat f_i(\vx)-f^\star_i(\vx)|\le\frac{\varepsilon_0}{100m}.
\end{equation}

By \eqref{eq:weighted-neuron}, each $v_{ij}\phi_{ij}(\vx)$ depends only on
that neuron's own initialization and noise. Thus, with the batch fixed,
the random variables $v_{ij}\phi_{ij}(\vx)$, $j\in[J]$, are i.i.d.
Also, \eqref{eq:useful-affine} and \eqref{eq:trained-activation} give
$\phi_{ij}(\vx)\le3m^2$ whenever $v_{ij}\ne0$, so
$|v_{ij}\phi_{ij}(\vx)|\le3m^2v_{\max}$. Cauchy--Schwarz and independence give
\[
 \begin{aligned}
 \E\bigl|\widehat f_i(\vx)-\E\widehat f_i(\vx)\bigr|
 &\le\sqrt{\operatorname{Var}(\widehat f_i(\vx))}
 =\frac1J\sqrt{\sum_{j=1}^J\operatorname{Var}(v_{ij}\phi_{ij}(\vx))}\\
 &\le\frac1J\sqrt{\sum_{j=1}^J\E[(v_{ij}\phi_{ij}(\vx))^2]}
 \le\frac{3m^2v_{\max}}{\sqrt J}.
 \end{aligned}
\]
Summing over the $m$ blocks and integrating over $\vx\sim\Dcal$ gives
\begin{equation}\label{eq:block-sampling-error}
 \E\sum_{i=1}^m\norm{\widehat f_i-\E\widehat f_i}_{L_1(\Dcal)}
 \le\frac{3m^3v_{\max}}{\sqrt J}\le\delta\varepsilon_0,
\end{equation}
where we used the width bound (on $J$) in Lemma~\ref{lem:hyperparameters} in the appendix.
By Markov's inequality and \eqref{eq:block-sampling-error}, with probability
at least $1-\delta$,
$\sum_{i=1}^m\norm{\widehat f_i-\E\widehat f_i}_{L_1(\Dcal)}<\varepsilon_0$.
Equation~\eqref{eq:block-mean-approximation} and the triangle inequality then give
\[
\sum_{i=1}^m\norm{\widehat f_i-f^\star_i}_{L_1(\Dcal)}<2\varepsilon_0.
\]
Define $\va^*$ by $a_0^*=-1$ and $a_{ij}^*=2v_{ij}/J$. Then
$\ip{\va^*}{\vphi}=-1+2\sum_{i=1}^m\widehat f_i$, so
\eqref{eq:hinge_bound} follows
by Lemma~\ref{lem:targets}, and
$\norm{\va^*}^2\le1+4mv_{\max}^2/J\le2$, where again we used the bound on $J$ from Lemma~\ref{lem:hyperparameters} in the appendix.
\end{proof}

It remains to show that Stage~2 finds an output vector almost as good as
$\va^*$. We combine the projected-SGD bound with a second-moment bound
valid for all neurons, including those not used by $\va^*$.

\begin{proofof}{Theorem~\ref{thm:main}}
A sign step
changes each hidden parameter by at most one, and each encoded block has $m^2$
active literals, so by \eqref{eq:initialization}, uniformly over
$\vx\in\Xcal_m$,
\[
 |\ip{\vw_{ij}^{(1)}}{\overline{\vx}_i}+b_{ij}^{(1)}|
 \le |b_{ij}^{(0)}|+m\norm{\vw_{ij}^{(0)}}+m^2+1
 \le 3m^2+1+m\norm{\vw_{ij}^{(0)}}.
\]
Since $\E\norm{\vw_{ij}^{(0)}}^2=2m^2\kappa^2$, the triangle inequality
in $L_2$ gives
\[
 \left(\E\sup_{\vx\in\Xcal_m}\phi_{ij}(\vx)^2\right)^{1/2}
 \le3m^2+1+\sqrt2\,m^2\kappa\le4m^2.
\]
The last inequality uses $\kappa=1/(16mM)$ and
$M^2\ge16(2m^2+1)$. Summing over the $mJ$ neurons and the constant
coordinate gives
\begin{equation}\label{eq:feature-moment}
 \E\sup_{\vx\in\Xcal_m}\norm{\vphi(\vx)}^2\le1+16m^5J.
\end{equation}
This bound holds whatever the Stage-1 gradient signs are, so it does not
require conditioning on any success event. 
Let $\mathcal E$ be the event that the conclusions of
Lemma~\ref{lem:compression} and Proposition~\ref{prop:features} both hold,
and set $q:=3\log(1/\varepsilon)$. Define $\widetilde{\va}$ to be equal to $q\va^*$ on $\mathcal E$ and zero otherwise.
The comparator depends only on the Stage-1 batch, initialization, and noise,
and is fixed conditional on these quantities. On $\mathcal E$, its norm satisfies
$\norm{\widetilde{\va}}\le\sqrt2q<R_{\mathrm{out}}$, and
$\log(1+e^{-qz})\le e^{-q}+q(1-z)_+$ gives
$\E_{\vx}\ell(h^\star(\vx),\ip{\widetilde{\va}}{\vphi(\vx)})\le e^{-q}+8q\varepsilon_0$.
On $\mathcal E^c$, its norm is zero and
$\ell(h^\star(\vx),\ip{\widetilde{\va}}{\vphi(\vx)})=\log2$.
Since $\Prob[\mathcal E^c]\le2\delta$ and $\varepsilon<1/4$,
\begin{equation}\label{eq:comparator-loss}
 \begin{aligned}
 \E\bigl[\E_{\vx}\ell(h^\star(\vx),\ip{\widetilde{\va}}{\vphi(\vx)})\bigr]
 &\le e^{-q}+8q\varepsilon_0+2\delta\log2
 =\varepsilon^3+\frac{24\varepsilon}{1000}
                  +\frac{2\varepsilon\log2}{1000}
                  \le\frac{\varepsilon\log2}{4}.
 \end{aligned}
\end{equation}

Stage~2 retains the initial output weights, so $\norm{\va^{(1)}}=1$.
Since $R_{\mathrm{out}}>1$, the squared distance to any comparator in the
projection ball is at most $4R_{\mathrm{out}}^2$.
Conditional on Stage~1, the population logistic loss is convex, and fresh
Stage-2 samples give unbiased gradients of norm at most $\norm{\vphi(\vx)}$.
Let $\va^{\mathrm{avg}}:=T^{-1}\sum_{t=1}^T\va^{(t)}$.
Applying the projected-SGD bound of
\citet[Eq.~(2.19)]{nemirovski2009robust}, averaging over Stage~1, and using
\eqref{eq:feature-moment}, we obtain
\[
 \begin{aligned}
 \E\Bigl[\E_{\vx}\bigl(
 \ell(h^\star(\vx),\ip{\va^{\mathrm{avg}}}{\vphi(\vx)})
 -\ell(h^\star(\vx),\ip{\widetilde{\va}}{\vphi(\vx)})\bigr)\Bigr]
 \le & \frac{2R_{\mathrm{out}}^2}{\eta T}+\frac\eta2(1+16m^5J) \\
\underset{(\star)}{\le} &\frac{11R_{\mathrm{out}}\sqrt{m^5J}}{\sqrt T}
 \underset{(\star\star)}{\le}\frac{\varepsilon\log2}{2},
 \end{aligned}
\]
where $(\star)$ uses the definition of $\eta$ and $(\star\star)$ follows from Lemma~\ref{lem:hyperparameters} in the appendix.
Together with \eqref{eq:comparator-loss}, 
\[
\E[\E_{\vx}\ell(h^\star(\vx),\ip{\va^{\mathrm{avg}}}{\vphi(\vx)})]
\le3\varepsilon\log2/4.
\]
A classification mistake incurs logistic loss at least $\log2$, so
\[
 \E\bigl[\Prob_{\vx}[\hat h(\vx)\ne h^\star(\vx)]\bigr]
 \le\frac{3\varepsilon}{4}<\varepsilon.
\]

Finally, we bound the computational cost.
Forming the literal encoding of each sample costs $O(d)\le O(P)$.
Computing the Stage-1 gradient costs $O(PB)$ arithmetic operations, and
Stage~2 costs $O(PT)$, including network evaluation, projection, and
averaging. Initialization and the hidden-layer update cost $O(P)$, so the
total cost is $O(Pn)$. Table~\ref{tab:params} gives
$\log B,\log J,\log T,\log P=O(m^{2/3}\log^2(m/\varepsilon))$.
Since $n=B+T$, this implies
$\log(Pn)=O(m^{2/3}\log^2(m/\varepsilon))$, proving \eqref{eq:resources}.
\end{proofof}

\section{Discussion}\label{sec:discussion}

Our result establishes that the computational advantage of learned
representations can persist in the distribution-free setting, resolving
the question of \citet{FKS}. Thus, dimension complexity cannot serve as a
universal runtime lower bound for natural gradient-based neural network
learning. In our construction, a single gradient step on the hidden layer
produces features sufficient for subsequent convex optimization of the
output layer.

A natural next question is whether ordinary SGD, without injected noise, can achieve a comparable distribution-free separation. Moreover, it would be interesting to understand for what classes of functions such separation results can be achieved.

\section*{AI Disclosure}
Generative AI assisted with writing and with formalizing mathematical claims.

\bibliography{references}
\appendix
\newpage

\section{Verification of the Hyperparameters}\label{app:parameters}

Table~\ref{tab:params} specifies Algorithm~\ref{alg:training} directly in terms
of $m$, $\varepsilon$, and the other training hyperparameters. Every
occurrence of $C_0$ refers to the same constant, fixed once for all $m$ and
$\varepsilon$. Lemma~\ref{lem:hyperparameters} in the appendix verifies the numerical
inequalities used in the proof; no proof
quantities need to be computed by the algorithm.

\begin{table}[ht]
\centering\small
\renewcommand{\arraystretch}{1.5}
\begin{tabular}{@{}ll>{$\displaystyle}l<{$}@{}}
\toprule
Hyperparameter & Role & \text{Choice}\\
\midrule
$B$ & Stage-1 batch size & \lceil C_0(m/\varepsilon)^4\rceil\\
$J$ & Neurons per block & \lceil\exp(C_0^3m^{2/3}\log^2(m/\varepsilon))\rceil\\
$\kappa$ & Initial weight scale & \frac1{C_0m^2\log(m/\varepsilon)}\\
$\sigma$ & Stage-1 noise standard deviation & \frac{m^2}{C_0\sqrt{\log B}}\\
$R_{\mathrm{out}}$ & Output-layer projection radius & C_0\log(1/\varepsilon)\\
$T$ & Stage-2 steps & \left\lceil\frac{C_0^3m^5J\log^2(1/\varepsilon)}{\varepsilon^2}\right\rceil\\
$\eta$ & Stage-2 step size & \frac{R_{\mathrm{out}}}{\sqrt{Tm^5J}}\\
\bottomrule
\end{tabular}
\caption{Hyperparameters of Algorithm~\ref{alg:training}, where $C_0>0$ is a universal constant chosen in the proof.}\label{tab:params}
\end{table}

\begin{lemma}[Hyperparameter verification]\label{lem:hyperparameters}
For every $m\ge2$ and $\varepsilon\in(0,1/4)$, the choices in
Table~\ref{tab:params}, with a sufficiently large universal $C_0$,
satisfy
\[
 \begin{aligned}
 M^2&\ge16(2m^2+1),\\
 B&\ge2s,\qquad
 B\ge\frac{2m}{\varepsilon_0}
       \left(s\log B+\log\frac{m(s+1)}\delta\right),\\
 \sqrt J&\ge\frac{3m^3v_{\max}}{\delta\varepsilon_0},\qquad
 \frac{4mv_{\max}^2}{J}<1,\\
 R_{\mathrm{out}}&\ge5\log(1/\varepsilon),\qquad
 \frac{11R_{\mathrm{out}}\sqrt{m^5J}}{\sqrt T}
 \le\frac{\varepsilon\log2}{2}.
 \end{aligned}
\]
Here the auxiliary quantities are those defined in Section~\ref{sec:proof}.
\end{lemma}
\begin{proof}
Since $M=(C_0/16)m\log(m/\varepsilon)$, for $C_0\ge960$ we have
$M\ge60m\log8>6m$, hence $M^2>36m^2\ge16(2m^2+1)$.
\paragraph{Batch size.}
The two batch-size inequalities make the union bound in
Lemma~\ref{lem:compression} at most $\delta$.
Since $s=O(m^{2/3})$ and
$\log(1/\varepsilon_0)=O(\log(m/\varepsilon))$, the right-hand side of the
second inequality is
\[
 O\bigl((1+\log C_0)m^{5/3}\varepsilon^{-1}\log^2(m/\varepsilon)\bigr),
\]
with an implicit constant independent of $C_0$.
The choice $B=\lceil C_0(m/\varepsilon)^4\rceil$ satisfies both inequalities
uniformly for $m\ge2$ and $\varepsilon<1/4$ once $C_0$ is large enough.

\paragraph{Width.}
The proof uses
\[
 k=O(m^{2/3}\log(m/\varepsilon)),\qquad
 \log\Lambda=O(m^{2/3}\log(m/\varepsilon)\log m),\qquad
 \log(B^{4C_0^2s})=4C_0^2s\log B.
\]
Here $M=(C_0/16)m\log(m/\varepsilon)$ and
$1/\kappa=C_0m^2\log(m/\varepsilon)$. The definitions of $\tau$ and
$v_{\max}$ therefore give
\[
 \log\left(\frac{3m^3v_{\max}}{\delta\varepsilon_0}\right)
 \le O\bigl(C_0^2(1+\log C_0)\bigr)\,m^{2/3}\log^2(m/\varepsilon),
\]
where the implicit constant is independent of $C_0$.
For sufficiently large $C_0$, the coefficient $C_0^3/2$ in $\log\sqrt J$
dominates this bound, proving the first width inequality. It also gives
$4mv_{\max}^2/J\le4\delta^2\varepsilon_0^2/(9m^5)<1$.

\paragraph{Output training.}
Substituting the choices of $T$ and $R_{\mathrm{out}}$ gives
$11R_{\mathrm{out}}\sqrt{m^5J}/\sqrt T\le11\varepsilon/\sqrt{C_0}$,
hence the stated bound for sufficiently large $C_0$.
The radius condition holds once $C_0\ge5$.
Equation~\eqref{eq:noise-upper-tail} only requires $C_0\ge16$.
Thus all requirements are compatible: choose $C_0$ large enough to
dominate these fixed constants and the fixed multiples of
$1+\log C_0$ required by the batch and width bounds.

\end{proof}

\section{Proofs Omitted From the Main Text}
\subsection{Proof of Lemma~\ref{lem:activation}}\label{app:deferred}

\begin{proofof}{Lemma~\ref{lem:activation}}
Conditional on the batch, the events $E_{ij}(\mathcal U)$, $j\in[J]$,
are mutually independent and independent of the initialization, since
each depends only on its neuron's noise, which is drawn independently
across neurons and independently of the initialization.
Since $r\le m^2/2$ and $m\ge2$,
$t_-\ge m^2/8$ and $t_+\le2m^2$.
By the choice of $\sigma$ and the Gaussian Chernoff bound,
\begin{equation}\label{eq:noise-upper-tail}
 \Prob[\xi>t_-]\le B^{-C_0^2/128}\le B^{-2},
\end{equation}
where the last inequality uses $C_0\ge16$.
Gordon's inequality \citep{Gordon} states that, for $Z\sim N(0,1)$ and $x>0$,
\begin{equation}\label{eq:gaussian-tail-lower}
 \Prob[Z>x]\ge\frac{x}{(x^2+1)\sqrt{2\pi}}e^{-x^2/2}.
\end{equation}
Set $x:=2C_0\sqrt{\log B}$. Since $\xi/\sigma\sim N(0,1)$ and
$t_+/\sigma\le x$, \eqref{eq:gaussian-tail-lower} gives
\[
 \Prob[\xi>t_+]\ge\Prob[Z>x]
 \ge\frac{x}{(x^2+1)\sqrt{2\pi}}B^{-2C_0^2}
 \ge\frac{B^{-2C_0^2}}{4C_0\sqrt{2\pi\log B}}
 \ge B^{-3C_0^2}.
\]
The last two inequalities use $x^2+1\le2x^2$ and
$1/(4C_0\sqrt{2\pi\log B})\ge B^{-C_0^2}$, respectively,
which hold for $C_0\ge16$ and $B\ge2$.
For $B\ge2$, the probability that all samples outside $\mathcal U$ are
inactive is at least $(1-B^{-2})^B\ge1/2$. Hence
\[
 p(\mathcal U)\ge\tfrac12 B^{-3C_0^2|\mathcal U|}
                  \ge B^{-4C_0^2s}.
\]
The independent output sign and bias interval have probabilities $1/2$
and $(r+3)/(4m^2)$. Also, Markov's inequality and
$M^2\ge16(2m^2+1)$ give
\[
 \Prob[\norm{\vw_{ij}^{(0)}}>\kappa M]
 \le\frac{2m^2}{M^2}\le\frac1{16}.
\]
Consequently
\[
 \Prob[(i,j)\text{ is }\mathcal U\text{-useful}]
 =p(\mathcal U)\frac{r+3}{8m^2}
   \Prob[\norm{\vw_{ij}^{(0)}}\le\kappa M]
 \ge\frac{r+3}{16m^2}B^{-4C_0^2s}.
\]
Since $s=\Theta(m^{2/3})$ and $B=\lceil C_0(m/\varepsilon)^4\rceil$,
we have $\log B=O(\log(m/\varepsilon))$, so the logarithm of this lower bound is
$-O(m^{2/3}\log(m/\varepsilon))$.
\end{proofof}

\subsection{Proof of Lemma~\ref{lem:compression}}\label{app:compression}

\begin{proofof}{Lemma~\ref{lem:compression}}
Fix $i\in[m]$ and start with $\mathcal U=\emptyset$ so that $\mathcal Z_i(\mathcal U)=[2m^2]$.
If there is an index $u\in[B]$ with $\psi_i(\vx^{(u)})=1$ and
$\cnt{\vx_i^{(u)}}{\mathcal Z_i(\mathcal U)}>r$, choose one such $u$ and
set $\mathcal U\leftarrow\mathcal U\cup\{u\}$. Repeat until no such index exists.
Each addition removes
at least $r+1$ literals, so at most $s=\lfloor2m^2/(r+1)\rfloor$ additions
are possible. All selected samples satisfy $\psi_i$, hence
$V_i\subseteq\mathcal Z_i(\mathcal U)$. At termination, every batch sample
satisfying $\psi_i$ has count at most $r$.

Fix $A\subseteq[B]$ with $|A|=j\le s$ and condition on
$(\vx^{(u)})_{u\in A}$. Then $\mathcal Z_i(A)$ is fixed,
and the inputs indexed outside $A$ remain independent draws from $\Dcal$.
If
$\Prob_{\vx}[\psi_i(\vx)=1,\ \cnt{\vx_i}{\mathcal Z_i(A)}>r]
>\varepsilon_0/m$,
the probability that no $u\notin A$ satisfies both
$\psi_i(\vx^{(u)})=1$ and
$\cnt{\vx_i^{(u)}}{\mathcal Z_i(A)}>r$
is at most $(1-\varepsilon_0/m)^{B-j}$.
A union bound over all $A$ with $|A|\le s$ and all blocks $i$
therefore bounds the probability that any constructed set violates
the probability condition in \eqref{eq:good-set} by
\[
 m\sum_{j=0}^s\binom Bj(1-\varepsilon_0/m)^{B-j}
 \le m(s+1)B^s\exp\left(-\frac{\varepsilon_0(B-s)}m\right)
 \le\delta.
\]
The last inequality follows from Lemma~\ref{lem:hyperparameters} in the appendix.
Together with $V_i\subseteq\mathcal Z_i(\mathcal U)$, this proves
that all constructed sets are good with probability at least $1-\delta$.
\end{proofof}

\subsection{Proof of Lemma~\ref{lem:polynomial}}\label{app:chebyshev}

We now construct the polynomials $Q_i$ of \eqref{eq:Q-props},
following \citet[Lemma~3.2]{DKR}. We include the construction because the
analysis needs an explicit degree and a bound on the whole cube.

\begin{proofof}{Lemma~\ref{lem:polynomial}}
By the choice of $k$, we have $k\ge\sqrt{r+1}\log(4m/\varepsilon_0)$.
Fix a block $i$. For $\vz\in\{0,1\}^{2m^2}$, since $r\ge2$, let
\[
 Q_i(\vz):=
 \frac{\Cheb_k((r+1-2\cnt{\vz}{V_i})/(r-1))}{\Cheb_k((r+1)/(r-1))},
\]
where $\Cheb_k$ is the Chebyshev polynomial of the first kind of degree
$k$. Then $Q_i(\vz)=1$ when $\cnt{\vz}{V_i}=0$. For
$1\le\cnt{\vz}{V_i}\le r$, we have
\[
 -1\le\frac{r+1-2\cnt{\vz}{V_i}}{r-1}\le1,
\]
so the numerator has magnitude at most one. For
the denominator, $\Cheb_k(\cosh t)=\cosh(kt)\ge e^{kt}/2$ and
\[
 \operatorname{arcosh}\frac{r+1}{r-1}
 =\log\frac{\sqrt r+1}{\sqrt r-1}\ge\frac2{\sqrt r},
\]
so the denominator is at least $e^{2k/\sqrt r}/2$, which exceeds
$m/\varepsilon_0$ by the choice of $k$. Hence
$|Q_i(\vz)|\le\varepsilon_0/m$ whenever $1\le\cnt{\vz}{V_i}\le r$.

On the whole cube, $0\le\cnt{\vz}{V_i}\le|V_i|\le m^2$, so the argument of the
Chebyshev polynomial has absolute value at most $2(2m^2+1)$. The bound
$|\Cheb_k(t)|\le(2\max\{1,|t|\})^k$ and the fact that the denominator is at
least one give $|Q_i|\le(4(2m^2+1))^{k+1}$.
Since $\cnt{\vz}{V_i}$ is linear in $\vz$, reducing powers $z_a^j$ to $z_a$ makes $Q_i$
a multilinear polynomial of degree at most $k$.
\end{proofof}

\subsection{Proof of Lemma~\ref{lem:gaussian}}\label{app:gaussian}

In this section, we use Fourier Analysis on the uniform cube.
As such, we may consider a uniform distribution on $\vu\in\{0,1\}^{m^2}$, in order to use Fourier Analysis which will in turn be used to prove \emph{pointwise} identities (holding for a fixed $\vu$).
To ease notation, we work directly with encoded blocks $\vz(\vu):=(\vu,\vone-\vu)$.

\begin{proofof}{Lemma~\ref{lem:gaussian}}
For $A\subseteq[m^2]$ let $\chi_A(\vu):=\prod_{a\in A}(2u_a-1)$.
These functions form an orthonormal basis on the uniform Boolean cube:
$\E_{\vu}[\chi_A(\vu)\chi_{A'}(\vu)]=\ind[A=A']$.
All Fourier coefficients and $L_2$ norms below use this uniform distribution.
In particular,
$\widehat Q(A):=\E_{\vu}[Q(\vz(\vu))\chi_A(\vu)]$ and
$\widehat\rho_{\vw}(A):=\E_{\vu}[\rho_{\vw}(\vz(\vu))\chi_A(\vu)]$.
Substituting $\vz(\vu)$ in $Q$ and reducing powers of Boolean variables
preserves its degree bound $k$, hence $\widehat Q(A)=0$ for $|A|>k$.

\paragraph{The kernel is diagonal in the Walsh basis.}
We first compute the expected product of two features from the same random
weight $\vw$. The Gaussian ReLU kernel formula \citep{ChoSaul} states that
for $\vw\sim N(\vzero,\kappa^2I_{2m^2})$ and any
$\vz,\vz'\in\R^{2m^2}$ with $\norm{\vz}=\norm{\vz'}=m$,
\begin{equation}\label{eq:cho-saul}
 \E_{\vw}[\relu(\ip{\vw}{\vz})\relu(\ip{\vw}{\vz'})]
 =\kappa^2m^2K\bigl(\ip{\vz}{\vz'}/m^2\bigr),
\end{equation}
where $K(t):=\bigl(\sqrt{1-t^2}+(\pi-\arccos t)t\bigr)/(2\pi)$.
All encoded blocks $\vz(\vu)$ have norm $m$. For $\vu,\vv\in\{0,1\}^{m^2}$, define
$\boldsymbol\omega(\vu,\vv)\in\{\pm1\}^{m^2}$ by
$\omega_a:=(2u_a-1)(2v_a-1)$, so $\omega_a=1$ if $u_a=v_a$ and
$\omega_a=-1$ otherwise. The inner product $\ip{\vz(\vu)}{\vz(\vv)}$ counts
the coordinates $a$ with $u_a=v_a$, so it equals
$\sum_a(1+\omega_a)/2=m^2t(\boldsymbol\omega)$, where
$t(\boldsymbol\omega):=\frac12+\frac1{2m^2}\sum_{a=1}^{m^2}\omega_a$.
Applying \eqref{eq:cho-saul} with $\vz=\vz(\vu)$ and $\vz'=\vz(\vv)$, and
dividing by $\kappa^2(2m^2+1)^2$, gives
\begin{equation}\label{eq:gaussian-kernel}
 \E_{\vw}[\rho_{\vw}(\vz(\vu))\rho_{\vw}(\vz(\vv))]
 =g(\boldsymbol\omega(\vu,\vv)),
 \qquad
 g(\boldsymbol\omega):=\frac{m^2}{(2m^2+1)^2}K(t(\boldsymbol\omega)).
\end{equation}
For $A\subseteq[m^2]$, write
$\boldsymbol\omega^A:=\prod_{a\in A}\omega_a$ and define
\begin{equation}\label{eq:gaussian-lambda}
 \lambda_A:=\E_{\boldsymbol\omega}[g(\boldsymbol\omega)\boldsymbol\omega^A],
\end{equation}
where $\boldsymbol\omega$ is uniform on $\{\pm1\}^{m^2}$.
Now fix $\vu$. The map $\vv\mapsto\boldsymbol\omega(\vu,\vv)$ is a bijection
from $\{0,1\}^{m^2}$ to $\{\pm1\}^{m^2}$, so it maps the uniform $\vv$ to
the uniform $\boldsymbol\omega$. Since $(2u_a-1)^2=1$, we have
$2v_a-1=\omega_a(2u_a-1)$, and therefore
$\chi_A(\vv)=\chi_A(\vu)\boldsymbol\omega^A$. Substituting the definition of
$\widehat\rho_{\vw}(A)$ and exchanging the expectation over $\vw$ with the
finite average over $\vv$ gives
\begin{equation}\label{eq:gaussian-covariance}
 \begin{aligned}
 \E_{\vw}[\widehat\rho_{\vw}(A)\rho_{\vw}(\vz(\vu))]
 &=\E_{\vv}\bigl[\chi_A(\vv)\,
   \E_{\vw}[\rho_{\vw}(\vz(\vv))\rho_{\vw}(\vz(\vu))]\bigr]\\
 &=\chi_A(\vu)\,\E_{\boldsymbol\omega}[\boldsymbol\omega^A
   g(\boldsymbol\omega)]
 =\lambda_A\chi_A(\vu).
 \end{aligned}
\end{equation}
The second equality uses \eqref{eq:gaussian-kernel} and the change of
variables $\vv\mapsto\boldsymbol\omega$, and the last uses
\eqref{eq:gaussian-lambda}.

\paragraph{Lower bounding $\lambda_A$.}
We first expand $K$ in a power series. Since
$\pi-\arccos t=\pi/2+\arcsin t$, we have
$K(t)=\frac t4+\frac1{2\pi}\bigl(\sqrt{1-t^2}+t\arcsin t\bigr)$. The
function $\sqrt{1-t^2}+t\arcsin t$ equals $1$ at $t=0$, and its derivative
is $\arcsin t=\sum_{j\ge1}\binom{2j-2}{j-1}t^{2j-1}/(4^{j-1}(2j-1))$.
Integrating this series term by term gives, for $t\in[-1,1]$ (at
$t=\pm1$ by continuity),
\begin{equation}\label{eq:gaussian-kernel-series}
 K(t)=\sum_{n\ge0}\mu_nt^n,
 \qquad
 \mu_0=\frac1{2\pi},\quad\mu_1=\frac14,\quad
 \mu_{2j}=\frac{\binom{2j-2}{j-1}}{2\pi\,4^{j-1}(2j-1)(2j)}\ (j\ge1),
\end{equation}
and $\mu_n=0$ for odd $n\ge3$. Since $\binom{2j-2}{j-1}\le4^{j-1}$, we
have $\mu_{2j}\le1/((2j-1)(2j))$, so the series converges absolutely on
$[-1,1]$. Moreover, every $\mu_n\ge0$, and $\mu_{2j}\ge1/(8(2j)^3)$ since
$\binom{2j-2}{j-1}\ge4^{j-1}/(2j-1)$ and $2\pi<8$.

Substituting \eqref{eq:gaussian-kernel-series} into the definition of $g$
in \eqref{eq:gaussian-kernel}, and then into \eqref{eq:gaussian-lambda},
gives
\begin{equation}\label{eq:gaussian-lambda-series}
 \lambda_A=\frac{m^2}{(2m^2+1)^2}\sum_{n\ge0}\mu_n\,
 \E_{\boldsymbol\omega}\bigl[t(\boldsymbol\omega)^n\boldsymbol\omega^A\bigr].
\end{equation}
Exchanging the series with the finite average over $\boldsymbol\omega$ is
valid because $t(\boldsymbol\omega)\in[0,1]$ and the series converges
absolutely there.

Since every $\mu_n\ge0$, to lower-bound $\lambda_A$ it suffices to show
that every expectation in \eqref{eq:gaussian-lambda-series} is
nonnegative, and then to lower-bound a single one. Expanding
$t(\boldsymbol\omega)^n=\bigl(\frac12+\sum_a\omega_a/(2m^2)\bigr)^n$ gives
a sum of products of the $\omega_a$ with positive coefficients, in which
a coordinate $\omega_a$ may appear several times. Since $\omega_a^2=1$,
each such product equals $\boldsymbol\omega^B=\prod_{a\in B}\omega_a$,
where $B\subseteq[m^2]$ is the set of coordinates appearing an odd number
of times in it. Collecting terms gives
$t(\boldsymbol\omega)^n=\sum_{B\subseteq[m^2]}c_B\boldsymbol\omega^B$,
where $c_B\ge0$ is the sum of the coefficients of the products that equal
$\boldsymbol\omega^B$. Orthonormality of the $\boldsymbol\omega^B$ gives
\begin{equation}\label{eq:gaussian-power-coefficient}
 \E_{\boldsymbol\omega}\bigl[t(\boldsymbol\omega)^n\boldsymbol\omega^A\bigr]
 =c_A\ge0.
\end{equation}
If $n\ge|A|$, the expansion contains, for each of the $|A|!$ orderings of
$A$, the term that takes $\omega_a/(2m^2)$ for the elements $a\in A$ from
the first $|A|$ factors in that order, and $\frac12$ from the remaining
factors. Each such term is $\boldsymbol\omega^A$ times
$2^{-(n-|A|)}(2m^2)^{-|A|}\ge(2m^2)^{-n}$, so $c_A\ge|A|!/(2m^2)^n$.

For nonempty $A$, we keep only the term in
\eqref{eq:gaussian-lambda-series} with $n$ the smallest even integer at
least $|A|$, so that $n\le|A|+1$. Using $\mu_n\ge1/(8n^3)$,
\begin{equation}\label{eq:gaussian-eigenvalue}
 \lambda_A
 \ge\frac{m^2}{(2m^2+1)^2}\cdot\frac{|A|!}{8n^3(2m^2)^n}
 \ge\frac{m^2}{128(2m^2+1)^2(2m^2)^n}
 \ge\frac1{128(2(2m^2+1))^{|A|+3}}.
\end{equation}
The second inequality uses $|A|!/n^3\ge|A|!/(|A|+1)^3\ge1/16$. The last
uses $(2m^2)^n\le(2(2m^2+1))^{|A|+1}$ and
$(2m^2+1)^2/m^2\le(2(2m^2+1))^2$.

\paragraph{An exact representation before truncation.}
Fourier inversion gives
\[
 Q(\vz(\vu))=\widehat Q(\emptyset)+\sum_{1\le|A|\le k}\widehat Q(A)\chi_A(\vu).
\]
For every nonempty $A$ in this sum, \eqref{eq:gaussian-eigenvalue}
ensures $\lambda_A>0$, so dividing \eqref{eq:gaussian-covariance} by
$\lambda_A$ gives
\[
 \E_{\vw}\left[
 \frac{\widehat\rho_{\vw}(A)}{\lambda_A}\rho_{\vw}(\vz(\vu))\right]
 =\chi_A(\vu).
\]
Thus, multiplying these coefficients by $\widehat Q(A)$ and summing
represents the nonconstant part of $Q$.
For the constant term $\widehat Q(\emptyset)$, the analogous coefficient
$\widehat\rho_{\vw}(\emptyset)/\lambda_\emptyset$ has positive mean, since
$\rho_{\vw}\ge0$, so it would violate property~(ii). Instead, we use
$\ip{\vw}{\vone}/(\kappa\sqrt{2m^2})$, whose mean is zero.
Expanding \eqref{eq:rho} and applying Gaussian integration by parts,
$\E_{\vw}[\ip{\vw}{\vone}f(\vw)]=\kappa^2\E_{\vw}[\ip{\vone}{\nabla f(\vw)}]$
with $f(\vw)=\relu(\ip{\vw}{\vz(\vu)})$, gives
\begin{equation}\label{eq:gaussian-constant}
 \begin{aligned}
 \E_{\vw}\left[
 \frac{\ip{\vw}{\vone}}{\kappa\sqrt{2m^2}}\rho_{\vw}(\vz(\vu))\right]
 &=\frac{\E_{\vw}[\ip{\vw}{\vone}\relu(\ip{\vw}{\vz(\vu)})]}
 {\kappa^2\sqrt{2m^2}(2m^2+1)}\\
 &=\frac{\ip{\vone}{\vz(\vu)}}{\sqrt{2m^2}(2m^2+1)}
   \Prob[\ip{\vw}{\vz(\vu)}>0]
 =\frac{m}{2\sqrt2(2m^2+1)},
 \end{aligned}
\end{equation}
where the last equality uses $\ip{\vone}{\vz(\vu)}=m^2$ and
$\Prob[\ip{\vw}{\vz(\vu)}>0]=1/2$.
Combining the coefficients for the constant and nonconstant terms, define
\begin{equation}\label{eq:gaussian-coefficient}
 \alpha_Q^0(\vw):=
 \frac{2\sqrt2(2m^2+1)}m\widehat Q(\emptyset)
 \frac{\ip{\vw}{\vone}}{\kappa\sqrt{2m^2}}
 +\sum_{1\le|A|\le k}
 \frac{\widehat Q(A)\widehat\rho_{\vw}(A)}{\lambda_A}.
\end{equation}
Equations~\eqref{eq:gaussian-covariance} and \eqref{eq:gaussian-constant}
then give the exact representation
\begin{equation}\label{eq:gaussian-exact-representation}
 \E_{\vw}[\alpha_Q^0(\vw)\rho_{\vw}(\vz(\vu))]=Q(\vz(\vu)).
\end{equation}

\paragraph{Bounding and normalizing the coefficient.}
To obtain property~(i), we bound $\alpha_Q^0$ and restrict it to
$\norm{\vw}\le\kappa M$.
For the nonconstant terms in \eqref{eq:gaussian-coefficient},
\eqref{eq:gaussian-eigenvalue} with $|A|\le k$, Cauchy--Schwarz, and
Parseval give
\[
 \begin{aligned}
 \left|\sum_{1\le|A|\le k}
 \frac{\widehat Q(A)\widehat\rho_{\vw}(A)}{\lambda_A}\right|
 &\le128(2(2m^2+1))^{k+3}
 \sqrt{\sum_A\widehat Q(A)^2}\sqrt{\sum_A\widehat\rho_{\vw}(A)^2}\\
 &=128(2(2m^2+1))^{k+3}\norm{Q\circ\vz}_{L_2}
 \norm{\rho_{\vw}\circ\vz}_{L_2}.
 \end{aligned}
\]
The constant term in \eqref{eq:gaussian-coefficient} is a product of three
factors, which we bound by $2\sqrt2(2m^2+1)/m\le4(2m^2+1)$,
$|\widehat Q(\emptyset)|\le\norm{Q\circ\vz}_{L_2}$, and
$|\ip{\vw}{\vone}|/(\kappa\sqrt{2m^2})\le\norm{\vw}/\kappa$.
Finally, $\norm{Q\circ\vz}_{L_2}\le\max|Q|\le(4(2m^2+1))^{k+1}$, and
$\norm{\rho_{\vw}\circ\vz}_{L_2}\le\norm{\vw}/\kappa$ because
$\rho_{\vw}(\vz(\vu))\le m\norm{\vw}/(\kappa(2m^2+1))$. Substituting these
bounds in \eqref{eq:gaussian-coefficient} gives
\begin{equation}\label{eq:gaussian-coefficient-bound}
 |\alpha_Q^0(\vw)|\le
 \left(4(2m^2+1)+128(2(2m^2+1))^{k+3}\right)
 (4(2m^2+1))^{k+1}\frac{\norm{\vw}}\kappa
 \le\frac\Lambda2\frac{\norm{\vw}}\kappa.
\end{equation}
For the last inequality, $m\ge2$ and $k\ge1$ give
$4(2m^2+1)\le m^6$, $2(2m^2+1)\le m^5$, and $128\le m^7$.
The numerical factor is therefore at most $m^{11k+29}\le m^{100k}/2=\Lambda/2$.
The sum in \eqref{eq:gaussian-coefficient} is over subsets of $[m^2]$;
if $k>m^2$, it includes every nonempty subset and the bound still holds.
By \eqref{eq:rho} and \eqref{eq:gaussian-coefficient}, for every $c>0$,
$\rho_{c\vw}(\vz)=c\rho_{\vw}(\vz)$ and
$\alpha_Q^0(c\vw)=c\alpha_Q^0(\vw)$.
Their product therefore scales by $c^2$. Since the Gaussian radius and
direction are independent, restricting the radius to
$\norm{\vw}\le\kappa M$ multiplies the expectation of this product by the
fraction of the second moment inside the ball,
\begin{equation}\label{eq:gaussian-radial-normalization}
 \nu:=\frac{\E_{\vw}[\norm{\vw}^2\ind[\norm{\vw}\le\kappa M]]}
 {\E_{\vw}\norm{\vw}^2}.
\end{equation}
Explicitly, write $\vw=R\vtheta$ with $R:=\norm{\vw}$ and
$\vtheta:=\vw/\norm{\vw}$. Then, for every $\vu$,
\begin{equation}\label{eq:gaussian-polar}
 \begin{aligned}
 \E_{\vw}\bigl[\alpha_Q^0(\vw)\rho_{\vw}(\vz(\vu))
 \ind[\norm{\vw}\le\kappa M]\bigr]
 &=\E[R^2\ind[R\le\kappa M]]\,
  \E_{\vtheta}[\alpha_Q^0(\vtheta)\rho_{\vtheta}(\vz(\vu))]\\
 &=\nu\,\E_{\vw}[\alpha_Q^0(\vw)\rho_{\vw}(\vz(\vu))],
 \end{aligned}
\end{equation}
where the last equality applies the same factorization without the
indicator.
The Gaussian moments $\E_{\vw}\norm{\vw}^2=2m^2\kappa^2$ and
$\E_{\vw}\norm{\vw}^4=2m^2(2m^2+2)\kappa^4$ give
\[
 1-\nu\le
 \frac{\E_{\vw}\norm{\vw}^4}{\kappa^2M^2\E_{\vw}\norm{\vw}^2}
 =\frac{2m^2+2}{M^2}\le\frac12,
\]
where the last inequality uses $M^2\ge16(2m^2+1)$ from
Lemma~\ref{lem:hyperparameters} in the appendix. Define
\begin{equation}\label{eq:gaussian-truncated-coefficient}
 \alpha_Q(\vw):=\frac{\alpha_Q^0(\vw)}\nu
 \ind[\norm{\vw}\le\kappa M].
\end{equation}
Since $\nu\ge1/2$, \eqref{eq:gaussian-coefficient-bound} gives
$|\alpha_Q(\vw)|\le\Lambda M$, and
\eqref{eq:gaussian-truncated-coefficient} makes it zero outside the ball.
This proves property~(i).

\paragraph{Exact centering after truncation.}
The Gaussian distribution restricted to a ball is rotationally
invariant, and all encoded blocks have norm $m$. Hence
$\E_{\vw}[\rho_{\vw}(\vz(\vu))\ind[\norm{\vw}\le\kappa M]]$ is constant in $\vu$.
For $A\ne\emptyset$, it follows that
\[
 \begin{aligned}
 \E_{\vw}[\widehat\rho_{\vw}(A)\ind[\norm{\vw}\le\kappa M]]
 &=\E_{\vu}\left[\chi_A(\vu)
 \E_{\vw}[\rho_{\vw}(\vz(\vu))\ind[\norm{\vw}\le\kappa M]]\right]=0.
 \end{aligned}
\]
Also, $\E_{\vw}[\ip{\vw}{\vone}\ind[\norm{\vw}\le\kappa M]]=0$ by symmetry.
Applying these identities to \eqref{eq:gaussian-coefficient} and
\eqref{eq:gaussian-truncated-coefficient} gives
$\E_{\vw}\alpha_Q(\vw)=0$, proving property~(ii).

\paragraph{Exact representation after truncation.}
By \eqref{eq:gaussian-truncated-coefficient},
\eqref{eq:gaussian-polar}, and
\eqref{eq:gaussian-exact-representation}, for every $\vu$,
\begin{equation}\label{eq:gaussian-truncated-representation}
 \begin{aligned}
 \E_{\vw}[\alpha_Q(\vw)\rho_{\vw}(\vz(\vu))]
 &=\frac1\nu\E_{\vw}\left[
 \alpha_Q^0(\vw)\rho_{\vw}(\vz(\vu))\ind[\norm{\vw}\le\kappa M]\right]\\
 &=\E_{\vw}[\alpha_Q^0(\vw)\rho_{\vw}(\vz(\vu))]
 =Q(\vz(\vu)).
 \end{aligned}
\end{equation}
This proves property~(iii).
\end{proofof}

\subsection{Proof of Lemma~\ref{lem:grid}}\label{app:grid}

\begin{proofof}{Lemma~\ref{lem:grid}}
Since $q,j$ are integers and $|\zeta|\le1/16$,
\[
 \relu(j-q+\zeta)-\relu(j-\tfrac12-q+\zeta)
 =\begin{cases}
 0 & j<q,\\
 \relu(\zeta) & j=q,\\
 \tfrac12 & j>q.
 \end{cases}
\]
Summing over $j=0,\ldots,r$ gives \eqref{eq:grid-identity}.
\end{proofof}

\end{document}